\documentclass[a4paper,fleqn]{cas-sc}

\usepackage[numbers, sort&compress]{natbib}
\usepackage{mathtools}
\usepackage{bbm}

\def\tsc#1{\csdef{#1}{\textsc{\lowercase{#1}}\xspace}}
\tsc{WGM}
\tsc{QE}

\newtheorem{theorem}{Theorem}
\newtheorem{lemma}[theorem]{Lemma}
\newtheorem{assumption}{Assumption}
\newdefinition{remark}{Remark}
\newproof{proof}{Proof}

\newcommand{\xstar}{x^*}
\newcommand{\RR}{\mathbb{R}}
\newcommand{\cc}{\mathfrak{c}}
\newcommand{\FF}{\mathcal{F}}
\newcommand{\EE}{\mathbb{E}}
\newcommand{\Cbar}{\overline C}
\newcommand{\Abar}{\overline A}
\newcommand{\Gbar}{\overline G}
\newcommand{\Bbar}{\overline B}
\newcommand{\PP}{\mathbb{P}}
\newcommand{\Mtilde}{\widetilde M}
\newcommand{\indi}{\mathbbm{1}}

\newcommand{\efrak}{\mathfrak{e}}

\begin{document}

\ExplSyntaxOn
\cs_set:Npn \__first_footerline: { }
\ExplSyntaxOff

\let\WriteBookmarks\relax
\def\floatpagepagefraction{1}
\def\textpagefraction{.001}

\shorttitle{Concentration and Mean-Square Bounds for Contractive SA}    

\shortauthors{Chandak}  

\title [mode = title]{Concentration and Mean-Square Bounds for Contractive Stochastic Approximation: A Unified Elementary Approach}  



%

\author[1]{Siddharth Chandak}[orcid=0000-0003-3237-7729]

\cormark[1]


\ead{chandaks@stanford.edu}



\affiliation[1]{organization={Department of Electrical Engineering, Stanford University},
            city={Stanford},
            postcode={94305}, 
            state={CA},
            country={USA}}







\cortext[1]{Corresponding author}


\begin{abstract}
We establish mean-square and concentration bounds for stochastic approximation (SA) with arbitrary norm contractive mappings, under a multiplicative noise model where the noise may scale affinely with the norm of the iterates, and the iterates are potentially unbounded. These settings arise in reinforcement learning, where operators are often contractive in the $\ell_\infty$ norm and the noise scales with the iterates. To address the arbitrary norm, earlier works replace the non-smooth squared norm with a smooth Lyapunov function constructed via the generalized Moreau envelope. For concentration analysis, these works handle multiplicative noise and unbounded iterates through a multi-stage bootstrapping argument that starts from a time-varying worst-case bound and iteratively refines it. We instead present a unified and elementary analysis that yields both bounds. Using an averaged noise sequence and corresponding auxiliary iterates, we obtain a one-step Lyapunov drift inequality for the normed error directly, without smoothing the norm or constructing an envelope. For the mean-square bound, we combine this drift inequality with an induction argument showing that the iterates remain bounded in expectation. For the concentration bound, we develop a probabilistic induction over a sequence of ``good'' events on which the iterates are controlled, allowing the standard Azuma-Hoeffding bound to be applied. Our approach yields the first sub-Gaussian tailed maximal (all-time) concentration bound for SA under multiplicative noise, by allowing the stepsize to depend logarithmically on the confidence level. Beyond the specific setting considered here, we discuss the generalizability of these proof techniques to other noise models and iterative algorithms.
\end{abstract}



\begin{keywords}
 stochastic approximation\sep mean-square error bound\sep concentration bound\sep multiplicative noise
\end{keywords}

\maketitle

\section{Introduction}
Stochastic approximation (SA) is a popular class of iterative algorithms used for finding the fixed point of an operator given its noisy realizations \citep{Kushner}. These algorithms have been widely studied over the past few decades due to their applications in fields such as reinforcement learning (RL), optimization, communication networks, and stochastic control \citep{Borkar-book}. While asymptotic convergence of SA has been widely studied in the literature, applications in optimization and RL have led to significant interest in obtaining finite-time guarantees, primarily in two forms: mean-square error bounds and high-probability (concentration) bounds.

Finite-time bounds are typically obtained via a Lyapunov drift inequality that quantifies how the squared norm of the error evolves at each step. This is straightforward for the Euclidean norm, whose square is smooth. However, RL applications often involve arbitrary norm contractions, e.g., Q-learning is contractive w.r.t.\ the max ($\ell_\infty$) norm \citep{Q-learning}, and SSP Q-learning is contractive w.r.t.\ the weighted max norm \citep{SSP}. As the squared norm is not smooth in these cases, obtaining a Lyapunov drift inequality is challenging. Another challenge is the multiplicative noise model, in which the noise scales affinely with the norm of the iterate rather than being uniformly bounded, and the iterates are potentially unbounded. Such noise models arise for algorithms such as $\text{TD}(0)$ for policy evaluation \citep{TD0}. This is especially difficult for concentration bounds, where the lack of almost-sure bounds on the iterates obstructs the use of standard martingale concentration inequalities.

Prior works handle these two challenges as follows. To obtain Lyapunov drift inequalities under arbitrary norm contractions, \citet{Zaiwei-Neurips} replace the non-smooth squared norm with a smooth Lyapunov function constructed via the generalized Moreau envelope. To deal with multiplicative noise and unbounded iterates in the concentration setting, \citet{Zaiwei-conc} develop a multi-stage bootstrapping argument that starts from a time-varying worst-case bound on the iterates and iteratively refines it, using moment-generating-function bounds and an exponential supermartingale with Ville's maximal inequality. 

We instead propose a unified and elementary framework that yields both the mean-square and concentration bounds. We describe our techniques and contributions in detail next.

\subsection{Our Contributions}
Our contributions are twofold: a unified and elementary analysis technique, and, as a consequence, the first sub-Gaussian tailed concentration bound for SA under multiplicative noise.
\begin{itemize}
    \item \textbf{Methodological contributions:} Our proof relies on two techniques: an averaged noise sequence together with corresponding auxiliary iterates to obtain a one-step Lyapunov drift inequality, and an inductive argument to handle the multiplicative noise and the lack of almost-sure bounds on the iterates. We briefly discuss the generalizability of these techniques to other noise models and iterative algorithms in Section~\ref{sec:conclusions}.
    \begin{itemize}
        \item \textbf{Averaged noise sequence and auxiliary iterates:} For the SA iteration with stepsize $\beta_k$, iterate $x_k$, and martingale difference noise $M_{k+1}$, we define an averaged noise sequence $\xi_{k+1} = (1-\beta_k)\xi_k + \beta_k M_{k+1}$ and auxiliary iterates $z_k = x_k - \xi_k$. This noise-averaging construction was introduced by \citet{Bravo} to obtain bounds for non-expansive SA. Here, we employ it to obtain a one-step Lyapunov drift inequality for the error $\|z_k - x^*\|_c^2$ directly, for an arbitrary norm $\|\cdot\|_c$ and fixed point $x^*$, without constructing an envelope. It then remains to bound $\|\xi_k\|_c$, which we do separately for the mean-square and concentration analyses using an inductive argument.

        \item \textbf{Inductive argument (mean-square):} For the mean-square bound, we use a straightforward induction showing that the iterates remain bounded in expectation, i.e., $\mathbb{E}\left[\|x_k\|_c^2\right]$ stays uniformly bounded for all $k$, which suffices to bound $\mathbb{E}\left[\|\xi_k\|_c^2\right]$.
    
        \item \textbf{Probabilistic induction (concentration):} The concentration bound is more delicate. As the iterates have no almost-sure bound, neither does the noise, which prevents a direct application of standard martingale concentration inequalities. To address this, we introduce a sequence of ``good'' events $\{C_k\}$, where $C_k$ is the event that the error is controlled at time $k$. Our goal is to upper bound the probability that these good events are violated even once, i.e., upper bound the probability $\PP\big(\bigcup_{k} \Cbar_k\big)$. We decompose this event by the first time a good event fails, writing $\bigcup_{k}\Cbar_k = \bigcup_{k}\big(\big(\bigcap_{i=0}^{k-1} C_i\big)\cap \Cbar_k\big)$, so that on each term all good events hold up to time $k-1$. Up to this time the iterates, and hence the noise, are bounded, allowing us to replace $M_{i+1}$ by the bounded sequence $\widetilde{M}_{i+1} = M_{i+1}\indi_{C_i}$ for $i\leq k-1$, and apply the Azuma--Hoeffding inequality. Summing over all time steps completes the argument.
    \end{itemize}
    Although the two settings use different inductive arguments, both proofs follow the same outline. First, a one-step drift inequality is established for the auxiliary iterates. Then, an inductive argument supplies the missing control on the iterates. Despite its simplicity, our approach does not sacrifice sharpness. For example, for the $\ell_\infty$ norm relevant to many RL applications, we obtain a rate of $O\left(\frac{\log(d)}{k}\right)$, matching \citet{Zaiwei-Neurips}. Using the same auxiliary iterates, but without the inductive arguments, we also recover mean-square and concentration bounds for the additive noise model (i.e., where the noise is uniformly bounded) in Appendix~\ref{app:additive}. 

    \item \textbf{Sub-Gaussian tailed concentration bound:} We obtain a high-probability bound of $O\left(\frac{\log(d(k+1)/\delta)}{k+h}\right)$ on $\|x_k-\xstar\|_c^2$, holding for all $k\geq 0$. This is the first sub-Gaussian tailed maximal (all-time) concentration bound for SA under the multiplicative noise model with possibly unbounded iterates. Here, ``maximal'' means that the bound holds simultaneously over the entire trajectory $\{x_k\}_{k\geq0}$, rather than at a single fixed time $k$. Obtaining this bound requires either the initial stepsize to be small, with $\beta_0 = O(1/\log(1/\delta))$. This dependence on $\delta$ is unavoidable: for a $\delta$-independent stepsize, \citet{Zaiwei-conc} obtain a bound scaling as $\log^m(1/\delta)$ with $m$ possibly greater than 1, and establish an impossibility result showing that a sub-Gaussian tail cannot be achieved in this setting. This comparison with \citet{Zaiwei-conc} is discussed in detail in Remark~\ref{remark-comparison-Zaiwei-conc}.

\end{itemize}
 
\subsection{Related Work}
We broadly divide the related work into two categories, based on the type of bound obtained.

\paragraph{Mean-square bounds.} While there has been significant work in recent years on finite-time mean-square error bounds for SA, much of it has focused on contraction with respect to the Euclidean norm. This includes works on linear SA \citep[e.g.,][]{SrikantYing, BhandariRussoSingal} and SGD \citep[e.g.,][]{MoulinesBach, nemirovski2009robust}. In contrast, we focus on works that deal with arbitrary norm contractions. \citet{Zaiwei-Neurips} were the first to obtain mean-square error bounds for SA with arbitrary norm contractions, using generalized Moreau envelopes. This technique was later extended to Markovian SA in \citet{Zaiwei-Markov} and to two-timescale SA in \citet{Chandak-CDC}. The generalized Moreau envelope approach across these works is presented as a general roadmap for finite-time analysis in the recent survey \citep{Lyapunov-survey}. In contrast, we do not use the generalized Moreau envelope, and instead employ an averaged noise sequence and auxiliary iterates to obtain a drift inequality for the normed error directly.

\paragraph{Concentration bounds.} There has been significant work on high-probability bounds for SA under the additive noise model, where the noise is uniformly bounded, including works on Q-learning \citep[e.g.,][]{Li-QL1, Li-QL2} and general SA \citep[e.g.,][]{ThoppeBorkar, Wierman, Chandak-conc}. A separate line of work studies concentration bounds for linear SA under multiplicative noise, particularly the $\text{TD}(0)$ algorithm: \citet{Dalal} obtain a bound that holds only from some time onward, \citet{Chandak-TD0} obtain a bound with a non-exponential tail, and \citet{Durmus} obtain a bound for SA with constant stepsizes. For general SA under multiplicative noise, the first maximal concentration bound was obtained in \citet{Zaiwei-conc}, which we compare against in detail in Remark~\ref{remark-comparison-Zaiwei-conc}. This has recently been extended to Markovian noise by \citet{Shangtong, Maguluri}. More recently, \citet{Pham} obtain time-uniform bounds at a level of abstraction that would require our noise-averaging construction and auxiliary iterates as an input to handle arbitrary-norm contractions. While they too use $\delta$-dependent stepsizes, this controls the initial unstable phase of their applications, and their theory does not cover the setting considered here.

In contrast to the works above, which rely on separate, specialized machinery for the mean-square and concentration analyses, we provide a unified and elementary analysis that yields both. While we recover the existing rates in the mean-square case, we obtain a sub-Gaussian tailed bound in the concentration case by working in the setting where the stepsize depends mildly on the confidence level.

\subsection{Outline and Notation}
The paper is organized as follows. Section~\ref{sec:formulation} sets up the problem and presents the mean-square and concentration bounds, followed by remarks comparing them with prior work and discussing the assumptions. Section~\ref{sec:proof-outline} gives a proof sketch of both bounds. Section~\ref{sec:conclusions} concludes with a discussion of the applicability of our proof techniques to settings beyond those considered here.

Throughout, $\|\cdot\|_c$ denotes an arbitrary norm on $\RR^d$, and $\|\cdot\|_2$ and $\|\cdot\|_\infty$ denote the Euclidean ($\ell_2$) and max ($\ell_\infty$) norms, respectively. For sequences, $O(\cdot)$ and $\Omega(\cdot)$ denote the standard order notation.

\section{Problem Formulation and Results}\label{sec:formulation}
Consider the following iteration.
\begin{equation}\label{SA-iter}
    x_{k+1}=x_k+\beta_k(f(x_k)-x_k+M_{k+1}).
\end{equation}
Here, $x_k\in\RR^d$ is the iterate and $\beta_k$ is the stepsize at time $k\geq 0$. The function $f(\cdot):\RR^d\to\RR^d$ denotes the mapping whose fixed point we wish to find, and $M_{k+1}$ denotes the martingale difference noise sequence. In this paper, we consider stepsize sequences of the form
\[\beta_k=\frac{\beta}{k+h},\]
where $\beta,h>0$. Analogous results for other choices of stepsize sequence can also be obtained using the techniques presented in this paper.

The first assumption states that the mapping $f(\cdot)$ is contractive with respect to some arbitrary norm $\|\cdot\|_c$.
\begin{assumption}\label{assu-contrac}
    The function $f(\cdot)$ is contractive with respect to norm $\|\cdot\|_c$, i.e., for all $x_1,x_2\in\RR^d$,
    \[\|f(x_1)-f(x_2)\|_c\leq \lambda\|x_1-x_2\|_c.\]
    Here $0\leq \lambda<1$ is the contraction factor.
\end{assumption}
Using the Banach contraction mapping theorem, the above assumption implies that $f(\cdot)$ has a unique fixed point. We denote this fixed point by $\xstar$, i.e., there exists a unique $\xstar$ such that $f(\xstar)=\xstar$. Our objective is to obtain mean-square and high-probability bounds on $\|x_k-\xstar\|_c$.

\subsection{Mean-Square Error Bound}
We first state the assumption on the noise sequence $\{M_{k+1}\}$ required for the mean-square error bound.
\begin{assumption}\label{assu-noise-MSE}
Define the family of $\sigma$-fields $\FF_k=\sigma(x_0, M_i, i\leq k)$. Then $\{M_{k+1}\}$ is a martingale difference sequence with respect to $\FF_k$, i.e.,
$\EE\left[M_{k+1}\mid \FF_k\right]=0.$
Moreover, for all $k\geq 0$,
\[\EE\left[\|M_{k+1}\|_c^2\mid \FF_k\right]\leq \sigma^2\left(1+\|x_k\|_c^2\right),\]
for some constant $\sigma>0$.
\end{assumption}
We next state our first key result, the mean-square error bound on the SA iteration.
\begin{theorem}\label{thm-MSE}
    Suppose Assumptions~\ref{assu-contrac} and \ref{assu-noise-MSE} are satisfied, and that $\beta\geq 2/(1-\lambda)$. Then there exist constants $\cc_1,\cc_2>0$ such that if $h\geq \cc_1$, then
    \[\EE\left[\|x_k-\xstar\|_c^2\right]\leq \frac{\cc_2}{k+h}.\]
\end{theorem}
Explicit values of constants $\cc_1$ and $\cc_2$ have been provided along with the theorem's proof in Appendix \ref{app:proof-MSE}. An outline of the proof is given in Section \ref{sec:proof-outline} through a series of lemmas. The following remarks discuss the assumptions on the noise and stepsize sequences, and present the order-wise rate for the case of the $\ell_\infty$ norm.
\begin{remark}[Multiplicative vs additive noise model]
    The conditional second moment of the noise is allowed to scale affinely with the squared norm of the iterates. This condition is commonly known as the \textit{multiplicative} noise model. It subsumes the \textit{additive} noise model, in which the conditional second moment of the noise is bounded by a constant. We present the additive model separately, along with its proof, in Appendix \ref{app:MSE-additive}, since the resulting bound differs only in constants but admits a much simpler argument.
\end{remark}
\begin{remark}[Conditions on the stepsize sequence]
    The first condition, $\beta\geq 2/(1-\lambda)$, is standard in the finite-time analysis of SA and is required for evaluating the recursions arising in the analysis (Lemma~\ref{lemma:aux2}). The second condition, that $h$ is large enough, is only needed to ensure that the stepsize is sufficiently small starting from time $k=0$. In the absence of this condition, our results still hold, but only from some time instant $k_0=\cc_1$ onwards.
\end{remark}
\begin{remark}[Rate for the $\ell_\infty$ norm]
    While the exact expression for $\cc_2$ is provided in Appendix \ref{app:proof-MSE}, we highlight the order-wise rate for the special case where $\|\cdot\|_c$ is the $\ell_\infty$ norm, given its relevance to RL applications. In this case, the rate is $O\!\left(\dfrac{\sigma^2\log d}{(1-\lambda)^3 k}\right)$, which matches the rate obtained in \citet{Zaiwei-Neurips}.
\end{remark}

\subsection{Concentration Bound}
We take the following assumption on the noise sequence required for the high-probability analysis.
\begin{assumption}\label{assu-noise-conc}
    $\{M_{k+1}\}$ is a martingale difference sequence with respect to $\FF_k$. Moreover, for all $k\geq 0$,
    \[\|M_{k+1}\|^2_c\leq \sigma^2 \left(1+\|x_k\|^2_c\right),\;\;\text{a.s.},\]
    for some constant $\sigma>0$.
\end{assumption}
The following theorem states the concentration bound, which is \textit{maximal} (or \textit{all-time}) in the sense that it bounds the entire trajectory of iterates $\{x_k\}$ simultaneously, rather than at a single fixed time $k$.
\begin{theorem}\label{thm-conc}
    Suppose Assumptions~\ref{assu-contrac} and \ref{assu-noise-conc} are satisfied, and that $\beta\geq 4/(1-\lambda)$. Then, for every $\delta\in(0,1)$, there exist constants $\cc_3=\Omega\left(\log\left(\frac{1}{\delta}\right)\right)$ and $\cc_4>0$ such that if $h\geq \cc_3$, then
    \[\PP\left(\forall k\geq 0:\; \|x_k-\xstar\|_c^2\leq \frac{\cc_4\log\frac{d(k+1)}{\delta}}{k+h}\right)\geq 1-\frac{\pi^2}{3}\delta.\]
\end{theorem}
Explicit values of constants $\cc_3$ and $\cc_4$ have been provided along with the theorem's proof in Appendix \ref{app:proof-conc}. An outline of the proof is given in Section \ref{sec:proof-outline} through a series of lemmas. The following remarks highlight the sub-Gaussian nature of the tail bound,  contrast it with existing results in the literature, and discuss results for the additive noise model.

\begin{remark}[Comparison with \citet{Zaiwei-conc}: sub-Gaussian tail via $h$ depending on $\delta$]\label{remark-comparison-Zaiwei-conc}
\citet{Zaiwei-conc} also establish a maximal (all-time) concentration bound for SA under the same multiplicative noise model, allowing the iterates $\{x_k\}$ to be unbounded, and using a stepsize sequence $\beta_k=\beta/(k+h)$ of the same form as ours. Their bound on the squared error scales as $\log^m(1/\delta)$, where $m\geq 1$ can be greater than 1 depending on the contraction factor and the noise magnitude. This corresponds to a Weibull tail rather than the sub-Gaussian tail ($m=1$) we obtain. In their formulation $h$ is fixed independently of the confidence level $\delta$. Under this restriction, they further show, via an impossibility result, that no choice of $\delta$-independent stepsize sequence can achieve a sub-Gaussian tail.

Our formulation matches theirs in every other respect: it is unconditional, holds for all $k\geq0$, allows for unbounded iterates, and permits multiplicative noise. The one place where our formulation differs is that we allow the stepsize to depend on $\delta$, with $\beta_0=O(1/\log(1/\delta))$. This falls outside the scope of their impossibility result, and is precisely what allows us to obtain, to our knowledge, the first sub-Gaussian tailed maximal concentration bound for SA under multiplicative noise. In other words, a sub-Gaussian tail and a $\delta$-independent stepsize sequence cannot be obtained simultaneously for SA with multiplicative noise. Our result shows that a mild dependence of the stepsize on the confidence level is sufficient to recover a sub-Gaussian tail.

Alternatively, if the stepsize is chosen independently of $\delta$, our proof technique still applies, but the bound only holds starting from some time instant $k_0=\Omega(\log(1/\delta))$ rather than from $k=0$. That is, a stronger guarantee (smaller $\delta$) requires either $h$ to grow with $\log(1/\delta)$, or, if $h$ is kept fixed, the bound to only take effect after a later time $k_0$.
\end{remark}

\begin{remark}[Additive noise model]
We present results for the additive noise model, where the noise sequence is almost surely bounded by a constant, in Appendix~\ref{app:conc-additive}. The main difference from the multiplicative noise case is that $h$ no longer needs to grow with $\log(1/\delta)$: a fixed $h$ suffices to obtain a sub-Gaussian tail, matching the result of \citet{Zaiwei-conc}.
\end{remark}

\begin{remark}[Rate for the $\ell_\infty$ norm]
As in Theorem~\ref{thm-MSE}, we highlight the order-wise rate of Theorem~\ref{thm-conc} for the special case where $\|\cdot\|_c$ is the $\ell_\infty$ norm. In this case, the bound reduces to
\(O\!\left(\frac{\sigma^2\log\left(d(k+1)/\delta\right)}{(1-\lambda)^3(k+h)}\right),\)
matching the mean-square rate of Theorem~\ref{thm-MSE} up to the additional $\log(1/\delta)$ factor expected of a high-probability bound.
\end{remark}

\section{Proof Outline}\label{sec:proof-outline}
In this section, we prove Theorems~\ref{thm-MSE} and \ref{thm-conc} through a sequence of lemmas. As discussed before, the first step of the proof is based on an averaged noise sequence and a corresponding set of auxiliary iterates, which allow us to derive a one-step Lyapunov drift inequality without requiring any additional machinery. 

Define the averaged noise sequence by $\xi_{k+1}=(1-\beta_k)\xi_k+\beta_kM_{k+1}$ with $\xi_0=0$. Unrolling this recursion gives us
\[\xi_k=\sum_{i=0}^{k-1}\beta_i\prod_{j=i+1}^{k-1}(1-\beta_j)M_{i+1}.\]
We also define the auxiliary iterates $z_k=x_k-\xi_k$. Our first lemma presents a rearrangement of our iteration \eqref{SA-iter} in terms of $z_k$, and presents the one-step Lyapunov drift inequality.
\begin{lemma}\label{lemma-drift}
Suppose Assumption \ref{assu-contrac} is satisfied and that $\beta_k\leq 1/2$ for all $k\geq 0$. Then,
    \begin{enumerate}[(a).]
        \item The iteration \eqref{SA-iter} can be rewritten as:
        \begin{equation}\label{SA-iter-zk}
            z_{k+1}=z_k+\beta_k(f(z_k)-z_k+\Delta_k),
        \end{equation}
        for all $k\geq 0$, where $\Delta_k=f(x_k)-f(z_k)$. Here, $\|\Delta_k\|_c$ is upper bounded by $\|\xi_k\|_c$.
        \item For all $k\geq 0$,
        \[\|z_{k+1}-\xstar\|_c^2\leq \left(1-\lambda'\beta_k\right)\|z_k-\xstar\|^2_c+\frac{3}{\lambda'}\beta_k\|\xi_k\|^2_c.\]
        Here, $\lambda'=1-\lambda$.
    \end{enumerate}
\end{lemma}

We note that the numerical constants in the final inequality above are loose and can be tightened for specific applications. To connect this lemma to our objective, note that
\[\|x_k-\xstar\|_c\leq \|z_k-\xstar\|_c+\|x_k-z_k\|_c\leq \|z_k-\xstar\|_c+\|\xi_k\|_c.\]
If a bound on $\|\xi_k\|_c$ is known, a bound on $\|z_k-\xstar\|_c$ can be obtained by solving the recursion in the above lemma. Thus, it remains to bound $\|\xi_k\|_c$. We derive the required bounds separately for the mean-square and high-probability analyses. We use an inductive argument in both cases to handle the multiplicative noise and the lack of almost-sure bounds on the iterates.

\subsection{Mean-Square Error Bound}
The following lemma gives a bound on $\EE\left[\|\xi_k\|_c^2\right]$.
\begin{lemma}\label{lemma:xi_k_bound}
    There exists a constant $\zeta_1(\|\cdot\|_c,d)$ depending only on the norm $\|\cdot\|_c$ and the dimension $d$, such that 
    \[\EE\left[\|\xi_k\|_c^2\right]\leq \zeta_1(\|\cdot\|_c,d)\sum_{i=0}^{k-1} \EE\left[\|M_{i+1}\|^2_c\right]\beta_i^2\prod_{j=i+1}^{k-1}(1-\beta_j).\]
\end{lemma}
We first comment on the constant $\zeta_1(\|\cdot\|_c,d)$. For ease of exposition, the proof of this lemma derives an explicit value of this constant using a simple norm-equivalence argument: we first pass from $\|\cdot\|_c$ to the Euclidean norm, use the orthogonality of the martingale difference sequence, and then pass back to $\|\cdot\|_c$. This method does not always yield the tightest bound. Sharper constants may be obtained by exploiting the martingale type-2 constant of the norm $\|\cdot\|_c$, which is typically determined for each norm separately. In Appendix \ref{app-proof-lemma-xi_k}, we provide such constants for some commonly used norms, including the $\ell_\infty$ norm.

If we were working with the additive noise model, the above lemma would suffice to establish the desired mean-square error bound. However, since we are working with the multiplicative noise model, an additional inductive argument is required to complete the proof.
\begin{lemma}\label{lemma:MSE-induction-step-1}
    Suppose the setting for Theorem \ref{thm-MSE} holds, and that $\EE\left[1+\|x_i\|_c^2\right]\leq \Gamma_1$ for all $i\leq k-1$ and some $\Gamma_1$, then
    \begin{enumerate}[(a).]
        \item For all $m\leq k$, \[\EE\left[\|\xi_m\|_c^2\right]\leq 2\sigma^2\Gamma_1\zeta_1(\|\cdot\|_c,d)\beta_m.\]
        \item \[\EE\left[\|z_k-\xstar\|_c^2\right]\leq \|x_0-\xstar\|^2_c\left(\frac{h}{k+h}\right)+\frac{12\sigma^2\Gamma_1\zeta_1(\|\cdot\|_c,d)}{\lambda'^2}\frac{\beta}{k+h}.\]
    \end{enumerate}
\end{lemma}
Using the one-step Lyapunov drift inequality from Lemma \ref{lemma-drift}, the above lemma shows that if the iterates are bounded in expectation till time $k-1$, then we can find a bound on the mean-square error at time $k$ as well. The remaining step in the proof is just identifying an appropriate $\Gamma_2$ such that $\EE\left[1+\|x_k\|_c^2\right]$ is bounded by $\Gamma_2$ for all $k\geq 0$.
\begin{lemma}\label{lemma:MSE-induction-step-2}
    Suppose the setting for Theorem \ref{thm-MSE} holds. Define \[\Gamma_2\coloneqq 2+8\|x_0-\xstar\|_c^2+4\|\xstar\|_c^2.\]
    If $\EE\left[1+\|x_i\|_c^2\right]\leq \Gamma_2$ for all $i\leq k-1$, then
    \begin{enumerate}[(a).]
        \item $\EE\left[\|x_k-\xstar\|_c^2\right]\leq \frac{\cc_2}{k+h}$.
        \item $\EE\left[1+\|x_k\|_c^2\right]\leq \Gamma_2$.
    \end{enumerate}
\end{lemma}
The above lemma shows that if $\EE\left[1+\|x_i\|_c^2\right]$ is bounded by $\Gamma_2$ for all $i\leq k-1$, then $\EE\left[1+\|x_k\|_c^2\right]$ is also bounded by $\Gamma_2$. Together with the trivial base case, this implies by strong induction that the bound holds for all $k\geq 0$. The first part of the lemma then implies that the required bound on the mean-square error also holds for all $k\geq 0$.

\subsection{Concentration Bound}
Under the additive noise model, the one-step Lyapunov drift inequality in Lemma \ref{lemma-drift}, together with the Azuma-Hoeffding inequality, is sufficient to obtain high-probability guarantees on the error. However, under the multiplicative noise model, the lack of an almost sure bound on the iterates prevents a direct application of standard martingale concentration inequalities. To address this issue, we first define the following three sequences of events:
\[A_k=\left\{\|z_k-\xstar\|^2_c\leq \frac{h}{k+h}\|x_0-\xstar\|_c^2+\frac{6}{\lambda'^2}\frac{\Gamma_3}{k+h}\log\frac{d(k+1)}{\delta}\right\},B_k=\left\{\|\xi_k\|^2_c\leq \frac{\Gamma_3}{k+h}\log\frac{d(k+1)}{\delta}\right\},\]
and 
\[C_k=\left\{\|x_k-\xstar\|^2_c\leq \frac{2h}{k+h}\|x_0-\xstar\|_c^2+\frac{14}{\lambda'^2}\frac{\Gamma_3}{k+h}\log\frac{d(k+1)}{\delta}\right\},\]
where $\Gamma_3\coloneqq 24\beta\zeta_2(\|\cdot\|_c,d)\sigma^2\left(1+2\|\xstar\|_c^2+4\|x_0-\xstar\|_c^2\right)$. Here, $\zeta_2(\|\cdot\|_c,d)$ plays the same role as the constant $\zeta_1(\|\cdot\|_c,d)$ in the mean-square error proof. In the high-probability analysis, however, we use norm equivalence with the $\ell_\infty$ norm to reduce the problem to coordinate-wise concentration bounds. The constant $\zeta_2(\|\cdot\|_c,d)$ comes from this norm-equivalence step, while the factor $\log(d)$ arises from the subsequent union bound over coordinates.

Our objective is to lower bound the probability of the event $\left(\bigcap_{k=0}^\infty C_k\right)$ or equivalently, to upper bound the probability of the event $\left(\bigcup_{k=0}^\infty \Cbar_k\right)$. The following lemma presents an expression for bounding this probability.
\begin{lemma}\label{lemma:conc-1}
    Suppose the setting for Theorem \ref{thm-conc} holds. Then, 
    \begin{enumerate}[(a).]
        \item $\bigcup_{k=0}^\infty \Cbar_k\subseteq \bigcup_{k=0}^\infty \left(\Abar_k \cup \Bbar_k\right)=\bigcup_{k=0}^\infty \Big(\left(\bigcap_{i=0}^{k-1}\left(A_i\cap B_i\right)\right)\bigcap \left(\Abar_k\cup\Bbar_k\right)\Big)$.
        \item $\PP\left(\bigcup_{k=0}^\infty \Cbar_k\right)\leq \sum_{k=0}^\infty \PP\left(\left(\bigcap_{i=0}^{k-1}\left(A_i\cap B_i\right)\right)\bigcap \Abar_k\right)+\sum_{k=0}^\infty \PP\left(\left(\bigcap_{i=0}^{k-1}\left(A_i\cap B_i\right)\right)\bigcap \Bbar_k\right).$
    \end{enumerate}
\end{lemma}
The above lemma decomposes the event $\bigcup_{k=0}^\infty \Cbar_k$ according to the first time instant at which one of the events $A_k$ or $B_k$ fails. This forms the basis of our probabilistic induction argument. Next, we bound the probabilities in the above lemma.
\begin{lemma}\label{lemma:conc-2}
    Suppose the setting for Theorem \ref{thm-conc} holds. Then for all $k\geq 0$,
    \begin{enumerate}[(a).]
        \item $\PP\left(\left(\bigcap_{i=0}^{k-1}\left(A_i\cap B_i\right)\right)\bigcap \Abar_k\right)=0$.
        \item \[\PP\left(\left(\bigcap_{i=0}^{k-1}\left(A_i\cap B_i\right)\right)\bigcap \Bbar_k\right)\leq \PP\left(\left\|\sum_{i=0}^{k-1}\beta_i\prod_{j=i+1}^{k-1}(1-\beta_j)\Mtilde_{i+1}\right\|^2_c>\frac{\Gamma_3}{k+h}\log\left(\frac{d(k+1)}{\delta}\right)\right),\]
        where $\Mtilde_{i+1}\coloneqq M_{i+1}\indi_{C_i}$ is a martingale difference sequence with respect to the filtration $\FF_i$.
    \end{enumerate}
\end{lemma}
The first part of the above lemma follows from the fact that, on the event $\left(\bigcap_{i=0}^{k-1} B_i\right)$, the bound defining $A_k$ is automatically satisfied and hence $\bigcap_{i=0}^{k-1} B_i\subseteq A_k$. The second part of the above lemma completes the key step in the inductive argument, since it allows us to apply standard martingale concentration inequalities such as the Azuma-Hoeffding inequality. The indicator in the definition of $\Mtilde_{i+1}$ ensures that the increment is nonzero only on the event $C_i$, where the iterate $x_i$ is controlled. Under the multiplicative noise assumption, this implies an almost sure bound on $\Mtilde_{i+1}$, allowing standard concentration inequalities to be applied. The following lemma bounds the probability in part (b) of the previous lemma. Summing these over all $k\geq 0$ completes the proof of Theorem \ref{thm-conc}.
\begin{lemma}\label{lemma:conc-3}
    Suppose the setting for Theorem~\ref{thm-conc} holds. Then for all $k\geq 0$,
    \[\PP\left(\left\|\sum_{i=0}^{k-1}\beta_i\prod_{j=i+1}^{k-1}(1-\beta_j)\Mtilde_{i+1}\right\|^2_c>\frac{\Gamma_3}{k+h}\log\left(\frac{d(k+1)}{\delta}\right)\right)\leq 2\frac{\delta}{(k+1)^2}.\]
\end{lemma}

\section{Conclusions and Generalizability of Techniques}\label{sec:conclusions}
We obtain mean-square and concentration bounds for SA with arbitrary norm contractive mappings and multiplicative noise using an elementary approach. We use a noise-averaging technique to obtain a one-step Lyapunov drift inequality on the normed error directly, and inductive arguments to handle the multiplicative noise and the lack of almost-sure bounds on the iterates. While our emphasis is on presenting a simpler proof technique, we also obtain the first sub-Gaussian tailed concentration bound for SA with multiplicative noise, using a stepsize that mildly depends on the confidence level.

Future directions involve extending our analysis to other noise models and iterative algorithms. A key feature of our approach is that the noise enters only through the averaged sequence $\xi_k$, which is better behaved and easier to analyze than the raw noise, as the averaging regularizes its effect. Moreover, our probabilistic induction is not specific to this setting: it applies to any iterative algorithm with unbounded iterates, with only the definition of the good events adapted to the algorithm. We outline these directions below.
\begin{itemize}
    \item \textbf{Other noise models:} Finite-time analysis of SA has typically focused on martingale and Markovian noise models, where the Markovian noise is often handled by decomposing it into a martingale difference sequence using the solutions to the Poisson equation \citep{benveniste2012adaptive}. Our approach extends naturally to more challenging noise models. This noise-averaging technique has recently been used to obtain finite-time bounds for strongly monotone SA under heavy-tailed and long-range dependent noise models \citep{Chandak-heavy}. Combined with our insight that this technique also handles arbitrary norm contractions, one could obtain finite-time bounds for algorithms such as Q-learning when the rewards are corrupted by heavy-tailed or correlated noise.
    \item \textbf{Other iterative algorithms:} Many iterative algorithms produce iterates that are not guaranteed to remain bounded. These include SSP Q-learning and RVI Q-learning for average reward MDPs, which are two-timescale and non-expansive SA algorithms, respectively \citep{SSP}. Our probabilistic induction argument can be used to obtain high-probability guarantees for such algorithms, following the same intuition.
\end{itemize}

\appendix

\section{Proofs required for mean-square error bound (Theorem \ref{thm-MSE})}\label{app:proof-MSE}

\subsection{Proof of Lemma \ref{lemma-drift}}
For all $k\geq 0$,
\begin{align*}
    &x_{k+1}=x_k+\beta_k(f(x_k)-x_k+M_{k+1})\\
    &\implies z_{k+1}+\xi_{k+1}=z_k+\xi_k+\beta_k(f(x_k)-z_k-\xi_k+M_{k+1})\\
    &\implies z_{k+1}+\xi_{k+1}=z_k+\beta_k(f(x_k)-z_k)+\xi_{k+1}.
\end{align*}
This gives us the required iteration for $z_k$: $z_{k+1}=z_k+\beta_k(f(z_k)-z_k+\Delta_k),$
where $\Delta_k=f(x_k)-f(z_k)$. Note that $\|\Delta_k\|_c=\|f(x_k)-f(z_k)\|_c\leq \lambda\|x_k-z_k\|_c=\lambda\|\xi_k\|_c$. Here, the inequality follows from the contractive nature of the map $f(\cdot)$ with respect to the norm $\|\cdot\|_c$. This completes the proof of part (a). 

For part (b), we first subtract $\xstar$ from the iteration for $z_k$ and use the fact that $\xstar$ is a fixed point for $f(\cdot)$.
\begin{align*}
    z_{k+1}-\xstar&=z_k-\xstar+\beta_k(f(z_k)-z_k+\Delta_k)\\
    &=(1-\beta_k)(z_k-\xstar)+\beta_k(f(z_k)-f(\xstar))+\beta_k\Delta_k.
\end{align*}
Then,
\begin{align*}
    \|z_{k+1}-\xstar\|_c&\leq (1-\beta_k)\|z_k-\xstar\|_c+\beta_k\lambda\|z_k-\xstar\|_c+\beta_k\|\Delta_k\|_c\\
    &\leq (1-(1-\lambda)\beta_k)\|z_k-\xstar\|_c+\beta_k\|\xi_k\|_c.
\end{align*}
Here, the first inequality follows from triangle inequality, the contractive nature of the map $f(\cdot)$, and from the fact that $0\leq \beta_k\leq 1$. Now,
\begin{align*}
    \|z_{k+1}-\xstar\|_c^2&\leq (1-\lambda'\beta_k)^2\|z_k-\xstar\|_c^2+\beta_k^2\|\xi_k\|_c^2+2(1-\lambda'\beta_k)\beta_k\|z_k-\xstar\|_c\|\xi_k\|_c.
\end{align*}
For the first term, we note that $(1-\lambda'\beta_k)^2=1-2\lambda'\beta_k+\lambda'^2\beta_k^2\leq 1-1.5\lambda'\beta_k$ under the condition that $\beta_k\leq 1/2\leq 1/(2\lambda')$. Now, for the third term above, we first note that $1-\lambda'\beta_k\leq 1$. Then, we apply the weighted AM-GM inequality ($2ab\leq \eta a^2+(1/\eta)b^2$) with $\eta=\lambda'/2, a=\|z_k-\xstar\|_c$, and $b=\|\xi_k\|_c$ to get
\begin{align*}
    2\beta_k\|z_k-\xstar\|_c\|\xi_k\|_c&\leq \frac{\lambda'}{2}\beta_k\|z_k-\xstar\|^2_c+\frac{2}{\lambda'}\beta_k\|\xi_k\|_c^2.
\end{align*}
Using the fact that $\beta_k\leq 1$, which implies that $\beta_k^2\leq \frac{1}{\lambda'}\beta_k$, we get
\[\|z_{k+1}-\xstar\|_c^2\leq \left(1-\lambda'\beta_k\right)\|z_k-\xstar\|_c^2+\frac{3}{\lambda'}\beta_k\|\xi_k\|^2_c.\]
This completes the proof of Lemma \ref{lemma-drift}.

\subsection{Proof of Lemma \ref{lemma:xi_k_bound}}\label{app-proof-lemma-xi_k}
Using the norm equivalence in finite-dimensional real spaces, there exists $\ell_{2,c}, u_{2,c}>0$, such that $\ell_{2,c}\|x\|_2\leq \|x\|_c\leq u_{2,c}\|x\|_2$ for all $x\in\RR^d$. This implies that 
\begin{align*}
    \EE\left[\|\xi_k\|_c^2\right]&\leq u_{2,c}^2\EE\left[\|\xi_k\|_2^2\right]\\
    &\leq u_{2,c}^2\EE\left[\left\|\sum_{i=0}^{k-1}\beta_i\prod_{j=i+1}^{k-1}(1-\beta_j)M_{i+1}\right\|_2^2\right]\\
    &\leq u_{2,c}^2\sum_{i=0}^{k-1}\left(\beta_i\prod_{j=i+1}^{k-1}(1-\beta_j)\right)^2\EE\left[\|M_{i+1}\|_2^2\right]\\
    &\leq \frac{u_{2,c}^2}{\ell_{2,c}^2}\sum_{i=0}^{k-1}\beta_i^2\prod_{j=i+1}^{k-1}(1-\beta_j)\EE\left[\|M_{i+1}\|_c^2\right].
\end{align*}
Here, the third inequality follows from the property that the terms of a martingale difference sequence are orthogonal. The fourth inequality follows from the fact that $0\leq \beta_k\leq 1$ for all $k\geq 0$. This completes the proof of Lemma \ref{lemma:xi_k_bound} with $u_{2,c}^2/\ell_{2,c}^2$ serving as a crude upper bound for $\zeta_1(\|\cdot\|_c,d)$. 

\noindent\textbf{Tighter bounds on $\zeta_1(\|\cdot\|_c,d)$ for $\ell_p$ and $\ell_\infty$ norms:}
Note that we wish to find $\zeta_1(\|\cdot\|_c,d)$ such that
\[\EE\left[\left\|\sum_{i}M_{i+1}\right\|_c^2\right]\leq \zeta_1(\|\cdot\|_c,d)\sum_i\EE\left[\|M_{i+1}\|_c^2\right],\]
holds for all $\RR^d$-valued martingale difference sequences $\{M_{k+1}\}$. When $\|\cdot\|_c$ is the $\ell_p$ norm for $2\leq p<\infty$, we use the fact that $\frac{1}{2}\|x\|_p^2$ is $(p-1)$-smooth w.r.t.\ $\|\cdot\|_p$ \citep[Example 5.11]{beck2017}. This implies that $\zeta_1(\ell_p,d)=(p-1)$ for $2\leq p<\infty$. Now, using H\"older's inequality, we have $\|x\|_\infty\leq \|x\|_p\leq d^{1/p}\|x\|_\infty$. Using this norm equivalence, we bound
\begin{align*}
    \EE\left[\left\|\sum_i M_{i+1}\right\|_\infty^2\right]&\leq \EE\left[\left\|\sum_i M_{i+1}\right\|_p^2\right]\leq \zeta_1(\ell_p,d)\sum_i \EE\left[\|M_{i+1}\|_p^2\right]\leq (p-1)\,d^{2/p}\sum_i \EE\left[\|M_{i+1}\|_\infty^2\right].
\end{align*}
This gives $\zeta_1(\ell_\infty,d)\leq (p-1)\,d^{2/p}$ for every $2\leq p<\infty$. Choosing $p=2\log d$ for $d\geq 3$
gives $d^{2/p}=d^{1/\log d}=e$, and hence $\zeta_1(\ell_\infty,d)\leq (2\log d-1)\,e$.

\subsection{Proof of Lemma \ref{lemma:MSE-induction-step-1}}
For all $m\geq 0$,
\begin{align*}
    \EE\left[\|\xi_m\|_c^2\right]&\leq \zeta_1(\|\cdot\|_c,d)\sum_{i=0}^{m-1}\beta_i^2\prod_{j=i+1}^{m-1}(1-\beta_j)\EE\left[\|M_{i+1}\|_c^2\right]\\
    &\leq \sigma^2\zeta_1(\|\cdot\|_c,d)\sum_{i=0}^{m-1}\beta_i^2\prod_{j=i+1}^{m-1}(1-\beta_j)\EE\left[1+\|x_i\|_c^2\right].
\end{align*}
The second inequality here follows from Assumption \ref{assu-noise-MSE}. Now, if $\EE\left[1+\|x_i\|_c^2\right]\leq \Gamma_1$ for all $i\leq k-1$, then for all $m\leq k$,
\begin{align*}
    \EE\left[\|\xi_m\|_c^2\right]&\leq \sigma^2\zeta_1(\|\cdot\|_c,d)\Gamma_1\sum_{i=0}^{m-1}\beta_i^2\prod_{j=i+1}^{m-1}(1-\beta_j)\\
    &\leq 2\sigma^2\zeta_1(\|\cdot\|_c,d)\Gamma_1\beta_m.
\end{align*}
Here, the last inequality follows from application of Lemma \ref{lemma:aux2} with $\phi=\beta, \epsilon=\beta^2$ and $\efrak=2$ and our assumption that $\beta\geq 2$. This completes the proof of part (a) of the lemma.

For part (b), unrolling the recursion from Lemma \ref{lemma-drift}, we get 
\begin{align*}
    \|z_k-\xstar\|_c^2\leq \prod_{i=0}^{k-1}\left(1-\lambda'\beta_i\right)\|z_0-\xstar\|_c^2+\frac{3}{\lambda'}\sum_{i=0}^{k-1}\beta_i\|\xi_i\|^2_c\prod_{j=i+1}^{k-1}\left(1-\lambda'\beta_j\right).
\end{align*}
Now, we note that $z_0=x_0$ by definition. Then, taking expectation and using the result from part (a), we get
\begin{align*}
    \EE\left[\|z_k-\xstar\|_c^2\right]\leq \|x_0-\xstar\|_c^2\prod_{i=0}^{k-1}\left(1-\lambda'\beta_i\right)+\frac{6\sigma^2\Gamma_1\zeta_1(\|\cdot\|_c,d)}{\lambda'}\sum_{i=0}^{k-1}\beta_i^2\prod_{j=i+1}^{k-1}\left(1-\lambda'\beta_j\right).
\end{align*}
Applying Lemma~\ref{lemma:aux1} with the condition that $\beta\geq1/\lambda'$ and $\lambda'\beta_k\leq 1$ for all $k$, and Lemma~\ref{lemma:aux2} with the condition that $\beta\geq2/\lambda'$ and $\lambda'\beta_k\leq 1$ for all $k$, we get
\begin{align*}
    \EE\left[\|z_k-\xstar\|_c^2\right]\leq \|x_0-\xstar\|_c^2\frac{h}{k+h}+\frac{12\sigma^2\Gamma_1\zeta_1(\|\cdot\|_c,d)}{\lambda'^2}\frac{\beta}{k+h}.
\end{align*}

\subsection{Proof of Lemma \ref{lemma:MSE-induction-step-2}}
We can apply the bounds from Lemma \ref{lemma:MSE-induction-step-1} with constant $\Gamma_1$ replaced by $\Gamma_2$.
\begin{align*}
    \EE\left[\|x_k-\xstar\|_c^2\right]&\leq 2\EE\left[\|z_k-\xstar\|_c^2+\|\xi_k\|_c^2\right]\\
    &\leq \frac{\cc_2}{k+h},
\end{align*}
where $\cc_2=2h\|x_0-\xstar\|_c^2+28\beta\sigma^2\Gamma_2\zeta_1(\|\cdot\|_c,d)/\lambda'^2$. For the second part of the proof, note that
\begin{align*}
    \EE\left[1+\|x_k\|_c^2\right]&\leq \EE\left[1+2\|x_k-\xstar\|_c^2+2\|\xstar\|_c^2\right]\\
    &\leq 1+2\|\xstar\|_c^2+4\|x_0-\xstar\|_c^2+\frac{56\beta\sigma^2\Gamma_2\zeta_1(\|\cdot\|_c,d)/\lambda'^2}{k+h}.
\end{align*}
Here, the second inequality follows from the fact that $h/(k+h)\leq 1$ for all $k\geq0$. Now, under the condition that $h\geq 112\beta\sigma^2\zeta_1(\|\cdot\|_c,d)/\lambda'^2$,
\begin{align*}
    \EE\left[1+\|x_k\|_c^2\right]&\leq 1+2\|\xstar\|_c^2+4\|x_0-\xstar\|_c^2+\Gamma_2/2=\Gamma_2.
\end{align*}
This completes the proof of Lemma \ref{lemma:MSE-induction-step-2}.

\subsection{Proof of Theorem~\ref{thm-MSE}}
Lemma \ref{lemma:MSE-induction-step-2} shows that if $\EE\left[1+\|x_i\|^2_c\right]\leq \Gamma_2$ for all $i\leq k-1$, then $\EE\left[1+\|x_k\|_c^2\right]\leq \Gamma_2$. For the base case of $k=0$, note that 
\[\EE\left[1+\|x_0\|_c^2\right]\leq 1+2\|x_0-\xstar\|_c^2+2\|\xstar\|_c^2\leq \Gamma_2.\]
Hence using the law of strong induction, $\EE\left[1+\|x_k\|_c^2\right]\leq \Gamma_2$ for all $k\geq 0$. This implies that for all $k \geq 0$, 
\[\EE\left[\|x_k-\xstar\|_c^2\right]\leq \frac{\cc_2}{k+h}.\]

\paragraph{\textbf{Values of constants in Theorem \ref{thm-MSE}:}} We assume that $h\geq \cc_1$, where $\cc_1=2\beta+112\beta\sigma^2\zeta_1(\|\cdot\|_c,d)/\lambda'^2.$ The value of $\cc_2$ in the bound is $\cc_2=2h\|x_0-\xstar\|_c^2+28\beta\sigma^2\Gamma_2\zeta_1(\|\cdot\|_c,d)/\lambda'^2$, where $\Gamma_2=2+4\|\xstar\|_c^2+8\|x_0-\xstar\|_c^2$.

\section{Proofs required for concentration bound (Theorem~\ref{thm-conc})}\label{app:proof-conc}
\subsection{Proof of Lemma \ref{lemma:conc-1}}
If $A_k\cap B_k$ holds, then
\begin{align*}
\|x_k-\xstar\|_c^2
&\leq 2\|z_k-\xstar\|_c^2+2\|\xi_k\|_c^2\\
&\leq \frac{2h}{k+h}\|x_0-\xstar\|_c^2
+\left(\frac{12}{\lambda'^2}+2\right)
\frac{\Gamma_3}{k+h}\log\frac{d(k+1)}{\delta}\\
&\leq \frac{2h}{k+h}\|x_0-\xstar\|_c^2
+\frac{14}{\lambda'^2}
\frac{\Gamma_3}{k+h}\log\frac{d(k+1)}{\delta},
\end{align*}
where the last inequality uses $\lambda'\leq 1$. Hence, $A_k\cap B_k\subseteq C_k$, or equivalently, $\Cbar_k\subseteq \Abar_k\cup\Bbar_k$. Taking a union over $k\geq 0$ gives
\(\bigcup_{k=0}^\infty \Cbar_k
\subseteq
\bigcup_{k=0}^\infty\left(\Abar_k\cup\Bbar_k\right).\)

For the second relation in part (a), let $\{E_k\}_{k\geq 0}$ be any sequence of events. We claim that
\(\bigcup_{k=0}^\infty E_k
=
\bigcup_{k=0}^\infty
\left(
\left(\bigcap_{i=0}^{k-1}E_i^c\right)\cap E_k
\right).\)
To see this, suppose $\omega\in \bigcup_{k=0}^\infty E_k$. Then there exists $k\geq 0$ such that $\omega\in E_k$. Let $k^\star$ be the smallest such index. Then $\omega\notin E_i$ for all $0\leq i\leq k^\star-1$, and hence
\(\omega\in \left(\bigcap_{i=0}^{k^\star-1}E_i^c\right)\cap E_{k^\star}.\)
Thus, $\omega$ belongs to the right-hand side. Conversely, if $\omega$ belongs to the right-hand side, then for some $k\geq 0$, we have $\omega\in E_k$, and hence $\omega\in \bigcup_{k=0}^\infty E_k$. Therefore, the equality holds. Taking $E_k=\Abar_k\cup\Bbar_k$ completes the proof of part (a).

Using part (a), we have
$\PP\left(\bigcup_{k=0}^\infty \Cbar_k\right)
\leq
\PP\left(\bigcup_{k=0}^\infty
\left(F_k\cap \left(\Abar_k\cup\Bbar_k\right)\right)\right),$
where $F_k\coloneqq \bigcap_{i=0}^{k-1}(A_i\cap B_i)$. Applying the union bound and using
$F_k\cap\left(\Abar_k\cup\Bbar_k\right)\subseteq
\left(F_k\cap\Abar_k\right)\cup\left(F_k\cap\Bbar_k\right)$ gives the desired inequality from part (b).

\subsection{Proof of Lemma \ref{lemma:conc-2}}
Unrolling the recursion from Lemma \ref{lemma-drift}, we get 
\begin{align*}
    \|z_k-\xstar\|_c^2\leq \prod_{i=0}^{k-1}\left(1-\lambda'\beta_i\right)\|z_0-\xstar\|_c^2+\frac{3}{\lambda'}\sum_{i=0}^{k-1}\beta_i\|\xi_i\|^2_c\prod_{j=i+1}^{k-1}\left(1-\lambda'\beta_j\right).
\end{align*}
On the event $\bigcap_{i=0}^{k-1} B_i$, $\|\xi_i\|_c^2\leq \frac{\Gamma_3}{i+h}\log\frac{d(i+1)}{\delta}$ for all $i\leq k-1$. Hence,
\begin{align*}
    \|z_k-\xstar\|_c^2&\leq \prod_{i=0}^{k-1}\left(1-\lambda'\beta_i\right)\|x_0-\xstar\|_c^2+\frac{3}{\lambda'}\sum_{i=0}^{k-1}\beta_i\frac{\Gamma_3}{i+h}\log\frac{d(i+1)}{\delta}\prod_{j=i+1}^{k-1}\left(1-\lambda'\beta_j\right)\\
    &\leq \prod_{i=0}^{k-1}\left(1-\lambda'\beta_i\right)\|x_0-\xstar\|_c^2+\frac{3\Gamma_3\beta}{\lambda'}\log\frac{d(k+1)}{\delta}\sum_{i=0}^{k-1}\frac{1}{(i+h)^2}\prod_{j=i+1}^{k-1}(1-\lambda'\beta_j).
\end{align*}
Applying Lemma \ref{lemma:aux1} and \ref{lemma:aux2} for the two terms above, under the condition that $\beta\geq 2/\lambda'$ and $\lambda'\beta_k\leq 1$ for all $k\geq 0$, we get
\begin{align*}
    \|z_k-\xstar\|_c^2\leq \frac{h}{k+h}\|x_0-\xstar\|_c^2+\frac{6\Gamma_3}{\lambda'^2(k+h)}\log \frac{d(k+1)}{\delta}.
\end{align*}
This implies that $\bigcap_{i=0}^{k-1}(A_i\cap B_i)\subseteq A_k$, and hence $\PP\left(\left(\bigcap_{i=0}^{k-1}\left(A_i\cap B_i\right)\right)\bigcap \Abar_k\right)=0$.

For part (b), note that
\begin{align*}
    &\PP\left(\left(\bigcap_{i=0}^{k-1}\left(A_i\cap B_i\right)\right)\cap \Bbar_k\right)\leq \PP\left(\left(\bigcap_{i=0}^{k-1}C_i\right)\cap \Bbar_k\right)\\
    &= \PP\left(\left(\bigcap_{i=0}^{k-1}C_i\right)\cap \left\{\left\|\sum_{i=0}^{k-1}\beta_i\prod_{j=i+1}^{k-1}(1-\beta_j)M_{i+1}\right\|_c^2> \frac{\Gamma_3}{k+h}\log\left(\frac{d(k+1)}{\delta}\right)\right\}\right)\\
    &\stackrel{(a)}{=} \PP\left(\left(\bigcap_{i=0}^{k-1}C_i\right)\cap \left\{\left\|\sum_{i=0}^{k-1}\beta_i\prod_{j=i+1}^{k-1}(1-\beta_j)M_{i+1}\indi_{C_i}\right\|_c^2> \frac{\Gamma_3}{k+h}\log\left(\frac{d(k+1)}{\delta}\right)\right\}\right)\\
    &\stackrel{(b)}{\leq} \PP\left(\left\|\sum_{i=0}^{k-1}\beta_i\prod_{j=i+1}^{k-1}(1-\beta_j)M_{i+1}\indi_{C_i}\right\|_c^2> \frac{\Gamma_3}{k+h}\log\left(\frac{d(k+1)}{\delta}\right)\right)\\
    &= \PP\left(\left\|\sum_{i=0}^{k-1}\beta_i\prod_{j=i+1}^{k-1}(1-\beta_j)\Mtilde_{i+1}\right\|_c^2> \frac{\Gamma_3}{k+h}\log\left(\frac{d(k+1)}{\delta}\right)\right).
\end{align*}

Here, equality $(a)$ follows from the fact that on the event $\bigcap_{i=0}^{k-1}C_i$, $\indi_{C_i}=1$ for all $i\leq k-1$, so the sum is unaffected by inserting the indicators. Inequality $(b)$ follows from $\PP(A\cap B)\leq \PP(B)$. The final equality uses the definition $\Mtilde_{i+1}\coloneqq M_{i+1}\indi_{C_i}$. This establishes the required bound, with $\Mtilde_{i+1}$ a martingale difference sequence with respect to the filtration $\FF_i$ since $\indi_{C_i}$ is $\FF_i$-measurable and $M_{i+1}$ is a martingale difference sequence with respect to $\FF_i$.

\subsection{Proof of Lemma \ref{lemma:conc-3}}
First note that $\left\|\Mtilde_{i+1}\right\|^2_c\leq \sigma^2(1+\|x_i\|_c^2)\indi_{C_i}\leq \sigma^2\left(1+2\|x_i-\xstar\|_c^2+2\|\xstar\|_c^2\right)\indi_{C_i}$. Hence, almost surely,
\begin{align*}
\left\|\Mtilde_{i+1}\right\|^2_c&\leq \sigma^2\left(1+2\|\xstar\|_c^2+\frac{4h}{i+h}\|x_0-\xstar\|_c^2+\frac{28}{\lambda'^2}\frac{\Gamma_3}{i+h}\log\left(\frac{d(i+1)}{\delta}\right)\right)\\
&\leq \sigma^2\left(1+2\|\xstar\|_c^2+4\|x_0-\xstar\|_c^2+\frac{28}{\lambda'^2}\frac{\Gamma_3}{i+h}\log\left(\frac{d(k+1)}{\delta}\right)\right)\\
&\eqqcolon r_i^2.
\end{align*}
For simplicity, define $w_{i+1}=\beta_i\prod_{j=i+1}^{k-1}(1-\beta_j)\Mtilde_{i+1}$ and $w_{i+1}^{(q)}$ denote the $q$-th component of the vector for $q\in\{1,\ldots, d\}$. Also define $v_k=\frac{\Gamma_3}{k+h}\log\left(\frac{d(k+1)}{\delta}\right)$. Then, we want to upper bound the probability of event $D_k$, where $D_k\coloneqq\left\{\left\|\sum_{i=0}^{k-1}w_{i+1}\right\|^2_c>v_k\right\}$.

Using the norm equivalence in finite-dimensional real spaces, there exists $\ell_{\infty,c}, u_{\infty,c}>0$, such that $\ell_{\infty,c}\|x\|_\infty\leq \|x\|_c\leq u_{\infty,c}\|x\|_\infty$ for all $x\in\RR^d$. Then,
\[D_k\subseteq \left\{\left\|\sum_{i=0}^{k-1}w_{i+1}\right\|_\infty^2>\frac{1}{u_{\infty,c}^2}v_k\right\}=\left\{\left|\sum_{i=0}^{k-1}w_{i+1}^{(q)}\right|^2>\frac{1}{u_{\infty,c}^2}v_k\;\text{for some}\; q\in\{1,\ldots,d\}\right\}.\]
Now, for all $q\in\{1,\ldots,d\}$ and $i\leq k-1$,
\[\left|w_{i+1}^{(q)}\right|^2\leq \left\|w_{i+1}\right\|^2_\infty\leq \beta_i^2\prod_{j=i+1}^{k-1}(1-\beta_j)\frac{r_i^2}{\ell_{\infty,c}^2},\;\;a.s.\]
Here, the final inequality uses the fact that $1-\beta_j\leq 1$ for all $j\geq 0$.
Using union bound and Azuma-Hoeffding bound, this implies that
\begin{align*}
    \PP(D_k)\leq 2d\exp\left(\frac{-\frac{1}{u_{\infty,c}^2}v_k}{2\sum_{i=0}^{k-1}\beta_i^2\prod_{j=i+1}^{k-1}(1-\beta_j)\frac{r_i^2}{\ell_{\infty,c}^2}}\right)\leq 2d\exp\left(\frac{-v_k}{2\zeta_2(\|\cdot\|_c,d)\sum_{i=0}^{k-1}\beta_i^2\prod_{j=i+1}^{k-1}(1-\beta_j)r_i^2}\right),
\end{align*}
where $\zeta_2(\|\cdot\|_c,d)=\frac{u_{\infty,c}^2}{\ell_{\infty,c}^2}$. For the denominator, we apply Lemma \ref{lemma:aux2} twice under the assumption that $\beta\geq 4$.
\begin{itemize}
    \item $\sum_{i=0}^{k-1}\beta_i^2\prod_{j=i+1}^{k-1}(1-\beta_j)\left(1+2\|\xstar\|_c^2+4\|x_0-\xstar\|_c^2\right)\leq 2\left(1+2\|\xstar\|_c^2+4\|x_0-\xstar\|_c^2\right)\beta_k$,
    \item $\sum_{i=0}^{k-1}\beta_i^2\prod_{j=i+1}^{k-1}(1-\beta_j)\frac{28\Gamma_3}{\lambda'^2(i+h)}\log\left(\frac{d(k+1)}{\delta}\right)\leq \frac{56\beta\Gamma_3}{\lambda'^2(k+h)^2}\log\left(\frac{d(k+1)}{\delta}\right)$.
\end{itemize}
These two inequalities give us
\begin{align*}
    \PP(D_k)\leq 2d\exp\left(\frac{-2\log\left(\frac{d(k+1)}{\delta}\right)}{8\zeta_2(\|\cdot\|_c,d)\sigma^2\left(\frac{\beta}{\Gamma_3}\left(1+2\|\xstar\|_c^2+4\|x_0-\xstar\|_c^2\right)+\frac{28\beta}{\lambda'^2(k+h)}\log\left(\frac{d(k+1)}{\delta}\right)\right)}\right)
\end{align*}
For $\Gamma_3=24\beta\zeta_2(\|\cdot\|_c,d)\sigma^2\left(1+2\|\xstar\|_c^2+4\|x_0-\xstar\|_c^2\right)$, we have
\[\frac{8\zeta_2(\|\cdot\|_c,d)\sigma^2\beta}{\Gamma_3}\left(1+2\|\xstar\|_c^2+4\|x_0-\xstar\|_c^2\right)\leq \frac{1}{3}.\]

Suppose $h\geq A\log A+1$ for some $A>0$. The inequality $\log(x)\leq x-1$, with $x=(k+1)/A$, implies
$A\log(k+1)\leq k+A\log A-A+1$. Therefore,
$k+h\geq A\log(k+1)$ for all $k\geq 0$, or equivalently,
$\frac{A\log(k+1)}{k+h}\leq 1$. Using our assumption that $h\geq A\log A+1$ for $A=\frac{672\zeta_2(\|\cdot\|_c,d)\beta\sigma^2}{\lambda'^2}$, we have
\[\frac{224\zeta_2(\|\cdot\|_c,d)\sigma^2\beta}{\lambda'^2(k+h)}\log(k+1)\leq \frac{1}{3},\]
for all $k\geq 0$. In addition, using the assumption that $h\geq A\log(d/\delta)$, we have
\[\frac{224\zeta_2(\|\cdot\|_c,d)\sigma^2\beta}{\lambda'^2(k+h)}\log\left(\frac{d}{\delta}\right)\leq\frac{1}{3}.\]
These together imply that 
\begin{align*}
    \PP(D_k)\leq 2d\exp\left(-2\log\left(\frac{d(k+1)}{\delta}\right)\right)\leq 2d\frac{\delta^2}{d^2(k+1)^2}\leq \frac{2\delta}{(k+1)^2}.
\end{align*}

\subsection{Proof of Theorem~\ref{thm-conc}}
Combining the results from Lemma \ref{lemma:conc-1}, \ref{lemma:conc-2}, and \ref{lemma:conc-3}, we get
\begin{align*}
    \PP\left(\bigcup_{k=0}^\infty \Cbar_k\right)\leq \sum_{k=0}^\infty 2\frac{\delta}{(k+1)^2}\leq \frac{\pi^2}{3}\delta.
\end{align*}
Hence, with probability at least $1-\frac{\pi^2}{3}\delta$, 
\[\|x_k-\xstar\|_c^2\leq \frac{\cc_4\log\frac{d(k+1)}{\delta}}{k+h} \text{ for all } k\geq 0.\]

\paragraph{\textbf{Values of constants in Theorem \ref{thm-conc}:}} We assume that $h\geq \cc_3$, where $\cc_3=2\beta+A\log\left(\frac{Ad}{\delta}\right)+1$, where $A=672\zeta_2(\|\cdot\|_c,d)\beta\sigma^2/\lambda'^2$. In the bound, \[\cc_4=\frac{2h\|x_0-\xstar\|_c^2}{\log(d/\delta)}+\frac{336\beta\zeta_2(\|\cdot\|_c,d)\sigma^2\left(1+2\|\xstar\|_c^2+4\|x_0-\xstar\|_c^2\right)}{\lambda'^2}.\]
Note that the first term in the definition of $\cc_4$ scales as $\frac{h}{\log(d/\delta)}$. This remains $O(1)$ for an appropriate choice of $h$: Theorem \ref{thm-conc} requires that $h\geq \cc_3=\Omega(\log(1/\delta))$. Choosing $h=\Theta(\cc_3)$ ensures that $h/\log(d/\delta)$ is bounded uniformly in $\delta$, so that $\cc_4=O(1)$ as $\delta\rightarrow 0$. This implies that our bound on $\|x_k-\xstar\|_c^2$ is still $O(\log(1/\delta))$.

\section{Additive Noise Model}\label{app:additive}
\subsection{Mean-Square Error Bound}\label{app:MSE-additive}
In the additive noise model, the noise sequence has a uniformly bounded conditional second moment. We state this formally in the following assumption.
\begin{assumption}\label{assu:MSE-additive-noise}
    The noise sequence $M_{k+1}$ is a martingale difference sequence with respect to the family of $\sigma$-fields $\FF_k$. Moreover, for all $k\geq 0$,
    \[\EE\left[\|M_{k+1}\|_c^2\mid\FF_k\right]\leq \sigma^2.\]
\end{assumption}
The following theorem states the mean-square error bound under this noise model.
\begin{theorem}\label{thm:MSE-additive}
    Suppose Assumptions~\ref{assu-contrac} and \ref{assu:MSE-additive-noise} are satisfied and the stepsize sequence is chosen such that $\beta\geq2/(1-\lambda)$ and $h\geq 2\beta$. Then, there exists $\cc_5>0$ such that
    \[\EE\left[\|x_k-\xstar\|_c^2\right]\leq \frac{\cc_5}{k+h}.\]
\end{theorem}
\begin{proof}
    Recall the bound on $\EE\left[\|\xi_k\|_c^2\right]$ from Lemma \ref{lemma:xi_k_bound}.
    \begin{align*}
        \EE\left[\|\xi_k\|_c^2\right]&\leq \zeta_1(\|\cdot\|_c,d)\sum_{i=0}^{k-1} \EE\left[\|M_{i+1}\|^2_c\right]\beta_i^2\prod_{j=i+1}^{k-1}(1-\beta_j)\\
        &\leq \sigma^2\zeta_1(\|\cdot\|_c,d)\sum_{i=0}^{k-1}\beta_i^2\prod_{j=i+1}^{k-1}(1-\beta_j)\\
        &\leq 2\sigma^2\zeta_1(\|\cdot\|_c,d)\beta_k.
    \end{align*}
Here, the second inequality follows from Assumption \ref{assu:MSE-additive-noise}, and the third inequality follows from application of Lemma \ref{lemma:aux2}, under the condition that $\beta_k\leq 1$ and $\beta\geq 2$. Now, applying this bound in the one-step Lyapunov drift inequality from Lemma \ref{lemma-drift} gives us for all $k\geq 0$,
\[\EE\left[\|z_{k+1}-\xstar\|_c^2\right]\leq (1-\lambda'\beta_k)\EE\left[\|z_k-\xstar\|_c^2\right]+\frac{6\sigma^2\zeta_1(\|\cdot\|_c,d)}{\lambda'}\beta_k^2.\]
Application of Lemma \ref{lemma:aux1} and Lemma \ref{lemma:aux2} under the condition that $\beta_k\leq 1$ and $\beta\geq 2/\lambda'$ gives us for all $k\geq 0$,
\[\EE\left[\|z_k-\xstar\|_c^2\right]\leq \|x_0-\xstar\|_c^2\frac{h}{k+h}+\frac{12\sigma^2\zeta_1(\|\cdot\|_c,d)}{\lambda'^2}\frac{\beta}{k+h}.\]
Finally, 
\begin{align*}
    \EE\left[\|x_k-\xstar\|_c^2\right]&\leq 2\EE\left[\|z_k-\xstar\|_c^2\right]+2\EE\left[\|\xi_k\|_c^2\right]\\
    &\leq \frac{2h\|x_0-\xstar\|_c^2+28\sigma^2\beta\zeta_1(\|\cdot\|_c,d)/\lambda'^2}{k+h},
\end{align*}
for all $k\geq 0$. This completes our proof for Theorem \ref{thm:MSE-additive} with $\cc_5=2h\|x_0-\xstar\|_c^2+28\sigma^2\beta\zeta_1(\|\cdot\|_c,d)/\lambda'^2$.\qed
\end{proof}

\subsection{Concentration Bound}\label{app:conc-additive}
We assume that the noise sequence is almost surely bounded.
\begin{assumption}\label{assu:conc-additive-noise}
    The noise sequence $M_{k+1}$ is a martingale difference sequence with respect to the family of $\sigma$-fields $\FF_k$. Moreover, for all $k\geq 0$,
    \[\|M_{k+1}\|^2_c\leq \sigma^2, \; a.s.\]
\end{assumption}
The following theorem states the concentration bound under this noise model.
\begin{theorem}\label{thm:conc-additive}
    Suppose Assumptions \ref{assu-contrac} and \ref{assu:conc-additive-noise} are satisfied and the stepsize sequence is chosen such that $\beta\geq2/(1-\lambda)$ and $h\geq 2\beta$. Then, there exist $\cc_6,\cc_7>0$ such that for $\delta\in(0,1)$,
    \[\PP\left(\forall k\geq 0:\; \|x_k-\xstar\|_c^2\leq \frac{\cc_6+\cc_7\log\frac{d(k+1)}{\delta}}{k+h}\right)\geq 1-\frac{\pi^2}{3}\delta.\]
\end{theorem}
\begin{proof}
    Define the following event:
    \[G_k\coloneqq\left\{\|\xi_k\|_c^2\leq 8\sigma^2\zeta_2(\|\cdot\|_c,d)\beta_k\log\frac{d(k+1)}{\delta}\right\}.\]
Following the same steps as in the proof of Lemma \ref{lemma:conc-3}, we get
\begin{align*}
    \PP(\Gbar_k)&\leq 2d\exp\left(\frac{-8\sigma^2\zeta_2(\|\cdot\|_c,d)\beta_k\log\frac{d(k+1)}{\delta}}{2\zeta_2(\|\cdot\|_c,d)\sigma^2\sum_{i=0}^{k-1}\beta_i^2\prod_{j=i+1}^{k-1}(1-\beta_j)}\right)\\
    &\leq 2d\exp\left(-2\log\left(\frac{d(k+1)}{\delta}\right)\right)\\
    &\leq 2d\frac{\delta^2}{d^2(k+1)^2}\leq 2\frac{\delta}{(k+1)^2}.
\end{align*}
Here, the second inequality follows from application of Lemma \ref{lemma:aux2} under the condition that $\beta_k\leq 1$ and $\beta\geq2$. By union bound, $\PP\left(\bigcup_{k=0}^\infty \Gbar_k\right)\leq \frac{\pi^2}{3}\delta$, and hence $\PP\left(\bigcap_{k=0}^\infty G_k\right)\geq 1-\frac{\pi^2}{3}\delta$.

Now, on the event $\bigcap_{k=0}^\infty G_k$, 
\begin{align*}
    \|z_{k}-\xstar\|_c^2&\leq \prod_{i=0}^{k-1}(1-\lambda'\beta_i)\|x_0-\xstar\|_c^2+\frac{24\sigma^2\zeta_2(\|\cdot\|_c,d)}{\lambda'}\log\left(\frac{d(k+1)}{\delta}\right)\sum_{i=0}^{k-1}\beta_i^2\prod_{j=i+1}^{k-1}(1-\lambda'\beta_j)\\
    &\leq \|x_0-\xstar\|_c^2\frac{h}{k+h}+\frac{48\sigma^2\zeta_2(\|\cdot\|_c,d)}{\lambda'^2}\frac{\beta}{k+h}\log\frac{d(k+1)}{\delta}.
\end{align*}
Here, the second inequality follows from application of Lemma \ref{lemma:aux1} and \ref{lemma:aux2} under the condition that $\beta_k\leq 1$ and $\beta\geq2/\lambda'$. 
Hence, with probability at least $1-\frac{\pi^2}{3}\delta$, 
\begin{align*}
    \|x_k-\xstar\|_c^2&\leq 2\|z_k-\xstar\|_c^2+2\|\xi_k\|^2_c\\
    &\leq \frac{2h\|x_0-\xstar\|_c^2+\frac{112\sigma^2\beta\zeta_2(\|\cdot\|_c,d)}{\lambda'^2}\log\frac{d(k+1)}{\delta}}{k+h}.
\end{align*}
This completes the proof for Theorem \ref{thm:conc-additive} with $\cc_6=2h\|x_0-\xstar\|_c^2$ and $\cc_7=112\sigma^2\zeta_2(\|\cdot\|_c,d)\beta/\lambda'^2$. \qed
\end{proof}

\section{Auxiliary Lemmas}
We present two lemmas which help us simplify the recursions typically obtained in finite-time analysis of SA, and are useful throughout this work.
\begin{lemma}\label{lemma:aux1}
Suppose $\phi_k=\phi/(k+h)$ for $\phi,h>0$. If $\phi\geq1$ and $\phi_k\leq 1$, then
\[\prod_{i=0}^{k-1}(1-\phi_i)\leq \frac{h}{k+h}.\]
\end{lemma}
\begin{proof}
    Using the fact that $1+x\leq e^x$ for all $x\in\RR$,
    \begin{align*}
        \prod_{i=0}^{k-1}(1-\phi_i)&=\prod_{i=0}^{k-1}\left(1-\frac{\phi}{i+h}\right)\leq \exp\left(-\phi\sum_{i=0}^{k-1}\frac{1}{i+h}\right)\leq \exp\left(-\sum_{i=0}^{k-1}\frac{1}{i+h}\right).
    \end{align*}
    Here the final inequality follows from our assumption that $\phi\geq 1$. Now, for any non-increasing function $g(x)$, we have that $\sum_{i=a}^{b}g(i)\geq \int_{a}^{b+1} g(x)dx.$ This implies that 
    \begin{align*}
        \sum_{i=0}^{k-1}\frac{1}{i+h}\geq \int_{0}^{k} \frac{1}{x+h}dx=\log\left(\frac{k+h}{h}\right).
    \end{align*}
    Finally, this implies that 
    \[\prod_{i=0}^{k-1}(1-\phi_i)\leq \frac{h}{k+h}.\]
This completes our proof.\qed
\end{proof}

\begin{lemma}\label{lemma:aux2}
    Let $\phi,h, \epsilon>0$. Suppose $\phi_k=\phi/(k+h)$. Let $\epsilon_k=\epsilon/(k+h)^{\efrak}$, where $\efrak>1$. If $\phi\geq 2(\efrak-1)$ and $\phi_k\leq 1$, then 
    \[\sum_{i=0}^{k-1}\epsilon_i\prod_{j=i+1}^{k-1}(1-\phi_j)\leq 2\frac{\epsilon_k}{\phi_k}.\]
\end{lemma}

\begin{proof}
    Define sequence $s_{0}=0$ and $s_{k+1}=(1-\phi_k )s_k+\epsilon_k$. Note that $s_k=\sum_{i=0}^{k-1}\epsilon_i\prod_{j=i+1}^{k-1}(1-\phi_j )$. We will use induction to show our required result. Suppose that $s_k\leq 2(\epsilon_k/\phi_k)$ holds for some $k$. Then,
    \begin{align*}
         2\frac{\epsilon_{k+1}}{\phi_{k+1}}-s_{k+1}&= 2\frac{\epsilon_{k+1}}{\phi_{k+1}}-(1- \phi_k)s_k-\epsilon_k\geq  2\frac{\epsilon_{k+1}}{\phi_{k+1}}-(1- \phi_k) 2\frac{\epsilon_k}{\phi_k}-\epsilon_k= 2\left(\frac{\epsilon_{k+1}}{\phi_{k+1}}-\frac{\epsilon_k}{\phi_k}\right)+\epsilon_k.
    \end{align*}
    Here the inequality follows from our assumption that the required inequality holds at time $k$.
    Now,
    \[\left(\frac{\epsilon_{k+1}}{\phi_{k+1}}-\frac{\epsilon_k}{\phi_k}\right)=\frac{\epsilon}{\phi}\left(\frac{1}{(k+h+1)^{\efrak-1}}-\frac{1}{(k+h)^{\efrak-1}}\right).\]
    For $\efrak-1>0$,
    \begin{align*}
        &\frac{1}{(k+h+1)^{\efrak-1}}-\frac{1}{(k+h)^{\efrak-1}}=\frac{1}{(k+h)^{\efrak-1}}\left(\left[\left(1+\frac{1}{k+h}\right)^{k+h}\right]^{-\frac{\efrak-1}{k+h}}-1\right)\\
        &\geq \frac{1}{(k+h)^{\efrak-1}}\left(e^{-\frac{\efrak-1}{k+h}}-1\right)\geq -\frac{1}{(k+h)^{\efrak-1}}\frac{\efrak-1}{k+h}=-\frac{\efrak-1}{\epsilon}\epsilon_k.
    \end{align*}
    Here, the first inequality follows from the inequality $(1+1/x)^x\leq e$ and $e^x\geq 1+x$ for all $x$. This implies 
    \begin{align*}
         2\frac{\epsilon_{k+1}}{\phi_{k+1}}-s_{k+1}&\geq - 2\frac{\epsilon}{\phi}\frac{\efrak-1}{\epsilon}\epsilon_k+\epsilon_k=\epsilon_k\left(1-\frac{2(\efrak-1)}{ \phi}\right).
    \end{align*}
    Since we have the assumption that $\phi\geq 2(\efrak-1)$, the following holds $s_{k+1}\leq  2\frac{\epsilon_{k+1}}{\phi_{k+1}}$. This completes the proof by induction.\qed
\end{proof}


\bibliographystyle{cas-model2-names}

\bibliography{refs}



\end{document}